\documentclass{article}

\PassOptionsToPackage{numbers, compress, sort}{natbib}
\usepackage[preprint]{neurips_2020}

\usepackage[utf8]{inputenc} % allow utf-8 input
\usepackage[T1]{fontenc}    % use 8-bit T1 fonts
\usepackage{hyperref}       % hyperlinks
\usepackage{url}            % simple URL typesetting
\usepackage{booktabs}       % professional-quality tables
\usepackage{amsfonts}       % blackboard math symbols
\usepackage{nicefrac}       % compact symbols for 1/2, etc.
\usepackage{microtype}      % microtypography
\usepackage{bm}
\usepackage{amsmath,amssymb,amsthm}
\usepackage{mathtools}
\usepackage{gensymb}
\usepackage{graphicx}
\usepackage[subrefformat=parens,labelformat=parens, position=bottom]{subfig}

\newtheorem{theorem}{Theorem}
\newtheorem{definition}{Definition}

\newtheorem{lemma}{Lemma}

\newcommand{\R}{{\rm I\!R}}

\newcommand{\Prob}{{\rm I\!P}}
\newcommand{\Exp}{{\mathop{\rm I\!E}}}

\newcommand{\one}[1]{{\bm{1}\left\{#1\right\}}}

\DeclarePairedDelimiter\floor{\lfloor}{\rfloor}

\title{How to Control the Error Rates of Binary Classifiers}

\author{%
  Milo{\v s} Simi{\' c}\thanks{milos.simic.ms@afrodita.rcub.bg.ac.rs} \\
  University of Belgrade\\
  Studentski trg 1, 11000 Belgrade\\
  \texttt{milos.simic.csci@gmail.com} \\
}

\begin{document}

\maketitle

\begin{abstract}
  The traditional binary classification framework constructs classifiers which may have good accuracy, but whose false positive and false negative error rates are not under users' control. In many cases, one of the errors is more severe and only the classifiers with the corresponding rate lower than the predefined threshold are acceptable. In this study, we combine binary classification with statistical hypothesis testing to control the target error rate of already trained classifiers. In particular, we show how to turn binary classifiers into statistical tests, calculate the classification $p$-values, and use them to limit the target error rate.
\end{abstract}
\section{Introduction}\label{sec:introduction}
A binary classifier, trained to distinguish between positive and negative classes, can produce two errors: a false positive and a false negative. In some cases, one of the errors is more severe and only the classifiers with the corresponding error rate lower than the predefined threshold are acceptable \cite{Turney1995}. For example, in medicine, a false negative means that a condition is missed and not treated, which could be fatal to the patient in question, whereas a false positive induces stress, but does not have such severe consequences \cite{Tong2020}. %Another shortcoming of traditional (non-probabilistic) classifiers is that they do not quantify uncertainty of their decisions, but only output labels. That is not acceptable when the end users want to know how confident the classifiers are in their predictions. 

In this study, we combine statistical hypothesis testing with binary classifiers and, in doing so, formulate a new classification framework capable of controlling both false positive and false negative error rates. More specifically, for each object that is to be classified, we show how to calculate the $p$-values associated with its membership in both positive and negative classes. We prove that classification $p$-values can be used to bound the error rate of our choice by the desired threshold. We also give probabilistic interpretations of the $p$-values and demonstrate usefulness of the proposed framework by applying it to the neural network for classifying distributions as normal or non-normal 
\cite{Simic2020} and deriving a new statistical test of normality from it.

The rest of the paper is organized as follows. Section \ref{sec:related_work} contains literature review. We develop the theory of our framework in Section \ref{sec:classification_tests}. The conducted experiments and obtained results are presented in Section \ref{sec:experiments}. Section \ref{sec:discussion_and_conclusion} contains discussion and draws conclusions. Proofs are in the appendix.

We use the following notation and conventions throughout the paper. Let $\mathcal{X}$ be the space of the objects we want to classify. Let $X$ be a random variable modeling a random object from $\mathcal{X}$, and let $x$ be a concrete object. Let $Y=y(X)$ and $y(x)$ denote the class to which $X$ and $x$ belong, where $0$ represents the negative, and $1$ is the label for the positive class. A classifier is a mapping $\hat{y}:\mathcal{X}\rightarrow\{0, 1\}$ defined as $\one{\tau(x) \geq t^*}$, where $\tau:\mathcal{X}\rightarrow \R$ is a scoring function that assigns higher scores to the objects for which the classifier is more confident that they belong to class $1$, $t^*$ is a decision threshold, and $\one{\cdot}$ is the indicator function. Note that for a random object $X$, it holds that $T=\tau(X)$, $Y=y(X)$, and $\hat{Y}=\hat{y}(X)$ are also random variables. 

Let $\Prob$, $\Exp$, and $V$ denote generic probability, expectation, and variance, respectively. The false positive ($FPR$) and false negative ($FNR$) rates of classifier $\hat{y}$ that uses $t^*$ as its decision threshold are defined as follows:
\begin{equation}
FNR(\hat{y},t^*) = \Prob\left(\hat{Y}=0\mid Y=1\right)\quad FPR(\hat{y},t^*) =\Prob\left(\hat{Y}=1\mid Y=0\right)
\end{equation}
We drop $\hat{y}$ or $t^*$ or both from $FNR(\hat{y},t^*)$ and $FPR(\hat{y}, t^*)$ to simplify notation when the classifier, the threshold or both are fixed or not relevant for exposition.

\section{Related Work}\label{sec:related_work}
The \textbf{cost-sensitive classification (CSC)} framework \cite{Habbema1974,Turney1995,Tapkan2016,Domingos1999,Elkan2001,Zadrozny2003,Zadrozny2001learning} has been the first to recognize the asymmetry between $FPR$ and $FNR$. CSC assigns different costs to the two errors. The direct CSC approach formulates new algorithms that take the costs into account from the start, e.g. \cite{Turney1995,Drummond2000,Ling2004,Tapkan2016}. 
The wrapper CSC approach turns cost-insensitive algorithms to cost-sensitive by manipulating the data before training \cite{Elkan2001,Zadrozny2003,Domingos1999}, or processing the outputs of classifiers whose training has been completed \cite{Witten2002,Chai2004,Sheng2006}. We refer readers to \cite{Sammut2010,Sheng2006} for a more detailed exposition of CSC. Both the direct and wrapper approaches have the same goal to minimize the total error cost. 

%However, sometimes it is not clear how to define the costs in the first place \cite{Tong2020}. For example, what should be cost of a false negative in medicine when we know that the error can have fatal consequences to the patient? It would be inappropriate and unethical to express the worth of one's life as a numerical or monetary cost. Moreover, in many cases, the goal is...
In many cases, the goal is not to minimize the total misclassification cost, but to limit the probability of the more severe error. Without loss of generality, assume that the false negative is the more serious error. The \textbf{Neyman-Pearson classification (NPC)} framework \cite{Tong2016} considers only the classifiers with $FNR \leq \alpha$, where $\alpha\in(0,1)$ is the predefined upper bound of $FNR$, and among those, aims to find the one with the minimal $FPR$. The empirical approach \cite{Cannon2002,Scott2005,Scott2007,ScottBellala2007,Han2008,Han2009,Rigollet2011,Barlaud2015} in NPC does so by solving the following optimization problem during training:
\begin{equation}
	\min_{\hat{y} \text{ s.t. } \widehat{FNR}(\hat{y})\leq \alpha + \varepsilon/2} \widehat{FPR}(\hat{y})
\end{equation}
where $\widehat{FNR}$ and $\widehat{FPR}$ are the empirical estimates of $FNR$ and $FPR$, calculated using the training data, and $\varepsilon>0$. The plug-in NPC approach \cite{Tong2013} directly applies the Neyman-Pearson Lemma \cite{Lehmann2006}. Let $q_0$ and $q_1$ represent the class-conditional densities over $\mathcal{X}$. The Neyman-Pearson Lemma states that the classifier with the minimal $FPR$ under the constraint that $FNR \leq \alpha$ is as follows:
\begin{equation}\label{eq:np_lemma}
	\hat{y}^*_{NP}(x) = \one{\frac{q_1(x)}{q_0(x)} \geq C_{\alpha}}\qquad \text{where } \Prob\left(\frac{q_1(X)}{q_0(X)} \geq C_{\alpha} \mid Y=0\right)=\alpha
\end{equation}
As $q_0$ and $q_1$ are not known, their estimates from the data, $\hat{q}_0$ and $\hat{q}_1$, are used instead in Equation \eqref{eq:np_lemma}, with $C_{\alpha}$ also estimated from the data. The original plug-in approach \cite{Tong2013} suffers from the curse of dimensionality as the densities are hard to approximate in high-dimensional spaces. \citet{Zhao2016} propose to do feature selection prior to density estimation, but their approach assumes that the features are independent from one another, which does not hold in general. Under the assumption that $q_0$ and $q_1$ are multivariate Gaussian densities with a common covariance matrix, \citet{Tong2020} formulate parametric Neyman-Pearson classifiers. Closely related to the NPC framework is the wrapper algorithm proposed by \citet{Tong2018}. The algorithm splits the original data into multiple training and test sets. Then, it trains a base classifier on each training set, and uses the test set to find the decision threshold closest to the value for which the target error rate is not greater than $\alpha$. An ensemble of such classifiers is proven to control the target rate with high probability, but comes with no guarantees that its non-target error rate is optimal \cite{Tong2018}. The wrapper can be used with any classification algorithm that has a scoring function.

\textbf{Cross-validated $p$-values} for multiclass classification were first formulated by \citet{Dumbgen2008}. We present them here for the case of binary classification. For an $x\in\mathcal{X}$, the $p$-value of hypothesis $H_c: y(x)=c$ ($c\in\{0,1\}$), denoted as $p_c(x)$, is defined as follows. Let $\mathcal{D}=\{(x_i,y_i)\}_{i=1}^m (y_i=y(x_i))$ be the training data and let $\mathcal{D}_i(x)$ denote the modified set $\mathcal{D}$ in which $x_i$ has been replaced with $x$. For each $i$ such that $y_i=c$, a classifier is trained on $\mathcal{D}_i(x)$. Let $m_c$ be the number of those new classifiers. For each such classifier, its score for $x_i$, denoted as $t_i(x_i)$, is recorded. Let $\tau(x)$ be the score that the original classifier, trained on the original set $\mathcal{D}$, outputs for $x$. The $p_c(x)$ is then a simple proportion:
\begin{equation}
	p_c(x) = \frac{1}{m_c+1}\left(1+\left\lvert\left\{i \in \{1,2,\ldots.m\}\,\middle|\, y_i=c \text{ and } \tau_i(x_i) \geq \tau(x) \right\}\right\rvert\right)
\end{equation}
This approach requires multiple classifiers to be trained whenever a new object needs to be classified. \citet{Dumbgen2008} provide a computationally less complex definition that requires training one new classifier on $\mathcal{D}\cup\{(x,c)\}$ for testing $H_c$. If $\tau$ is its scoring function, this $p_c$-value is:
\begin{equation}
p_c(x) = \frac{1}{m_c+1}\left(1+\left\lvert\left\{i \in \{1,2,\ldots.m\}\,\middle|\, y_i=c \text{ and } \tau(x_i) \geq \tau(x) \right\}\right\rvert\right)
\end{equation}
Although only $2$ additional classifiers are trained for each new object in the latter approach, it is still computationally very expensive. Both definitions are capable of controlling the target error rate and do not depend on the prior probabilities $\Prob(Y=0)$ and $\Prob(Y=1)$, but only on the available data.

\textbf{Typicality indices (typicality $p$-values)} \citet{McLachlan1992} introduce an alternative definition of $p$-values. Let $q_0$ and $q_1$ be the class-conditional densities as in the NPC framework. The typicality index of $x$ for class $c$ is:
\begin{equation}
p_c(x) = \Prob\left(q_c(X) \leq q_c(x) \mid y(X)=c\right)
\end{equation}
Although originally defined for the case where $q_0$ and $q_1$ are Gaussian, the typicality $p$-values can control the error rates no matter the exact forms of class-conditional distributions. However, just as the NPC framework, they too require density estimation, which is computationally challenging in high-dimensional settings.

The \textbf{statistical-test-as-a-proxy classification (STPC)} framework classifies objects by conducting classical two-sample statistical tests \cite{Liao2007,Ghimire2012,Guo2019,He2019}. Let $\mathcal{S}_0$ and $\mathcal{S}_1$ be the sets of scores for negative and positive training (or test) objects, respectively. If $x$ is positive, than the difference between $\mathcal{S}_0\cup\{\tau(x)\}$ and $\mathcal{S}_1$ should be smaller than that between $\mathcal{S}_0$ and $\mathcal{S}_1\cup\{\tau(x)\}$. The differences are quantified by the $p$-values of the tests that check whether the two sets come from the same distribution \cite{Liao2007,Ghimire2012}, or the tests which compare the distances between and inside those pairs of sets \cite{He2019,Guo2019}. The problem with this framework is that the $p$-values are small whenever datasets compared are large, in which case this approach may not be very useful.

\section{Deriving Tests from Binary Classifiers}\label{sec:classification_tests}
Similarities between classification and hypothesis testing are not hard to see. The null and alternative hypotheses in statistics correspond to $H_0: y(x)=0$ and $H_1: y(x) = 1$ in binary classification. The sample on which we test the hypotheses in statistics corresponds to the object that we want to classify in machine learning.  However, what is the classification equivalent of the test statistic?

Our idea is to treat the score as a random variable $T=\tau(X)$ and use it as the test statistic for which we can define the $p$-values. In standard statistical tests, the $p$-value both quantifies statistical uncertainty and controls the rate at which the test rejects a true null.
\begin{definition}[\cite{Abell1999,Biau2009}]\label{def:p_vrednost_verovatnoca_ekstremnijeg}
Let $T$ be the test statistic and $x$ a sample. Let higher values of $T(x)$ be more incompatible with the null hypothesis $H_0: \theta = \theta_0$, where $\theta$ is the parameter being tested and $\theta_0$ its hypothesized value. Let $X$ denote a random sample.  Then, the $p$-value of $x$ is:
\begin{equation}\label{eq:pv_eks}
p(x) = \Prob\left(T(X) \geq T(x) \mid H_0\right)
\end{equation}
\end{definition}
The $p$-values for classification can be defined analogously if we use $T=\tau(X)$ as the test statistic and assume that its higher values are interpreted as more compatible with $H_1$, and lower with $H_0$, which is the case with many classification algorithms, e.g. support vector machines.
\begin{definition}[Classification $p$-values]\label{def:classification_p_values} Let $T$ be the scoring statistic of the classifier whose scoring function is $\tau$. Let $x$ be a new object to classify. The classification $p$-value for hypothesis $H_0$, referred to as the $p_0$-value henceforth, is defined as follows:
		\begin{equation}
		p_0(x) = P_0\left(T \geq \tau(x) \mid Y = 0\right)
		\end{equation}
The classification $p$-value for hypothesis $H_1$, referred to as the $p_1$-value hereafter, is defined as follows:
		\begin{equation}
		p_1(x) = P_1\left(T \leq \tau(x) \mid Y = 1\right)
		\end{equation}
\end{definition}
The $p_c$-values follow the uniform distribution $U[0,1]$.
\begin{theorem}\label{th:p_vrednost_je_uniformna}
	It holds that $p_c(X) \sim U[0, 1]$ when $y(X)=c$ ($c\in\{0,1\}$). Therefore:
	\begin{equation}
	\Prob\left(p_c(X) \leq \alpha \mid y(X)=c\right) = \alpha\quad \text{for each }\alpha\in(0, 1)
	\end{equation}
\end{theorem}
This means that classifying $x$ as a negative object when $p_1(x)\leq \alpha$ results in $FNR$ being exactly $\alpha$. If our goal is to control $FPR$, we classify $x$ as positive when $p_0(x) \leq \alpha$. 
\paragraph{Probabilistic interpretation} The $p_c$-values belong to frequentist statistics and should be interpreted in line with statistical frequentism. The value of $p_1(x)$ should be understood as the long-term frequency at which the classifier $\one{\tau(\cdot) > \tau(x)}$, equivalent to $\one{p_1(\cdot) > p_1(x)}$, gives false positives. The philosophical justification for accepting $H_0$ when $p_1(x) \leq \alpha$ is that we would rarely be wrong in doing so \cite{Mayo2006} when $\alpha$ is small. In statistics, the most common values for $\alpha$ are $0.01$ and $0.05$, but $\alpha$ should be chosen depending on the application requirements. The same conclusions hold for $p_0(x)$. Conversely, the value of $1-\alpha$ is the long-term frequency of correctly classifying positive (or negative) objects as such using $\one{p_1(\cdot) > \alpha }$ (or $\one{p_0(\cdot) > p_1(x)}$).

Calculating the $p_c$-values and classifying objects by comparing those values to $\alpha$ recalibrates the original classifier. If $t$ is its decision threshold, $a(t)$ its accuracy, $w_1$ the prior probability $\Prob(y(X)=1)$, and $w_0=\Prob(y(X)=0)$, it is straightforward to see the following connection between $FPR(t)$, $FNR(t)$, and $a(t)$:
\begin{equation}
w_1 FNR(t) + w_0 FPR(t) + a(t) = 1
\end{equation}
Decreasing one of the rates to $\alpha$, e.g. $FNR$, essentially changes the threshold $t$ to some new value $t'$. The decrease in $FNR$ is compensated by distributing the difference $w_1(FNR(t)-\alpha)$ between $w_0 FPR(t')$ and $a(t')$. Likewise, increasing $FNR$ to $\alpha$ would cause $FPR$ and $a$ to shrink. If we plot $FNR$ against $FPR$, we will get the ROC--curve rotated around the line $TPR=1/2$ ($TPR$, true positive rate, equal to $1-FNR$). The points where the rotated ROC-curve intersects the line $FPR\mapsto(1-a)/w_0 - (w_1/w_0)FPR$ for a fixed accuracy $a\in[0,1]$ identifies the threshold scores which have the same accuracy, but different trade-offs between $FNR$ and $FPR$.

\subsection{Estimation of the $p_c$-values}
We propose to use the following non-parametric estimators of the $p_c$-values.
\begin{definition}\label{def:aproksimatori_klasifikacione_p_vrednosti}
Let $\mathcal{Z}=\left\{(X_i,Y_i)\right\}_{i=1}^{n}$ be a random set of objects which were not used during training the underlying classifier. Let $\mathcal{Z}_c = \{X_i^{(c)}\}_{i=1}^{n_c} = \left\{ X_i \mid i \in \{1,2,\ldots,n\}, Y_i = c \right\}$ ($c=0,1$). Let $x$ be a new object to classify. We use the following estimator for $p_0(x)$:
\begin{equation}\label{eq:aproksimator_p0}
\hat{p}_{0, n_0}(\mathbf{x}) = \frac{1}{n_0} \sum_{i=1}^{n_0}\bm{1}\left\{T(X_i^{(0)}) \geq T(\mathbf{x})\right\}
\end{equation}
and the following for $p_1(x)$:
\begin{equation}\label{eq:aproksimator_p1}
\hat{p}_{0, n_1}(\mathbf{x}) = \frac{1}{n_1} \sum_{i=1}^{n_1}\bm{1}\left\{T(X_i^{(1)}) \leq T(\mathbf{x})\right\}
\end{equation}
If $n_c$ is clear from the context, we will omit it to simplify notation. 
\end{definition}
It is important that $\mathcal{Z}_c$ does not contain objects that were used during training, so that their scores and the score of any random object that is to be classified are independent and identically distributed, as it is required for proofs.
 
The estimators \eqref{eq:aproksimator_p0} and \eqref{eq:aproksimator_p1} actually represent the estimators of the CDFs of $T(X) \mid y(X)=0$ and $1-T(X) \mid y(X)=1$. Theorem \ref{th:dobra_svojstva_aproksimatora} presents their properties.
\begin{theorem}\label{th:dobra_svojstva_aproksimatora}
Let $\{X_1,X_2,\ldots,X_{n}\}$ be an i.i.d sample of objects from class $c$ that we use to estimate $\hat{p}_{c,n}$. Then, $\hat{p}_{c, n}$ has the following properties:
\begin{enumerate}
		\item $\hat{p}_{c,n}$ is unbiased. \begin{equation}\Exp_{X_1,X_2,\ldots,X_{n}}\left[\hat{p}_{c,n}(\mathbf{x})\right]=p_c(\mathbf{x})\quad (\forall \mathbf{x})(y(\mathbf{x})=c)
		\end{equation}
		\item Its variance is finite, bound from above and drops linearly with $n$.
		\begin{equation}\label{eq:varijabilnost}
		V_{X_1,X_2,\ldots,X_{n}}\left[\hat{p}_{c,n}(x)\right] = \frac{1}{n}p_c(x)(1-p_c(x)) \leq \frac{1}{4n} \quad (\forall x)(y(x)=c)
		\end{equation}
		\item Its convergence to $p_c$ is almost sure.
		\begin{equation}\label{eq:skoro_sigurna_konvergencija}
		\Prob\left(\lim_{n\rightarrow \infty}\hat{p}_{c,n}=p_c\right) = 1
		\end{equation}
		\item It is consistent.
		\begin{equation}\label{eq:konzistentnost}
		\lim_{n\rightarrow\infty}\Prob\left(\lvert\hat{p}_{c,n}-p_c\rvert > \varepsilon \right) = 0\quad (\forall\varepsilon)(\varepsilon > 0)
		\end{equation}
		\item The confidence intervals of any width $\varepsilon>0$ can be obtained if $n$ is sufficiently large.
		\begin{equation}\label{eq:intervali_poverenja}
		\Prob\left(\sup_{x}\lvert\hat{p}_{c,n}(x)-p_c(x)\rvert\leq \varepsilon\right) \geq 1 - 2e^{-2n\varepsilon^2}
		\end{equation}
	\end{enumerate}
\end{theorem}
Theorem \ref{th:dobra_svojstva_aproksimatora} shows that if we use sufficiently large $\mathcal{Z}_c$ to calculate $\hat{p}_{c,n_c}$, the calculated estimates will not differ much from the exact $p_c$-values. Hence, the estimators practically behave as uniform random variables, which means that we can use them to control the chosen error rate with high probability. That is not a problem if we have a lot of data. If we do not, the following theorem gives some guarantees about the target error rate being bound by $\alpha$ from above if $\mathcal{Z}_c$ contains more than $1/\alpha$ elements.
\begin{theorem}\label{th:ogranicenost}
	Let $n>1/\alpha$. Let $\mathcal{Z}_c = \{X_1,X_2,\ldots,X_n\}$ be a random sample of objects whose true label is $c$ and which were not used during training. Let $X$ model a random object from class $c$. Then 
	\begin{equation}\label{eq:aproksimator_je_ogranicen}
	\Prob(\hat{p}_c(X) \leq \alpha) \leq \alpha
	\end{equation}
\end{theorem}
Since $X_1,X_2,\ldots,X_n$ are also random in Theorem \ref{th:ogranicenost}, we must carefully interpret it. The theorem does not imply that \eqref{eq:aproksimator_je_ogranicen} holds for each choice of $\mathcal{Z}_c$. Instead, it says that we can keep the error rate below the desired threshold $\alpha$ if we sample a different $\mathcal{Z}_c$ every time we want to classify a new object. For example, if we have only $50$ positive elements available after training, and $5\%$ is an acceptable upper bound for $FNR$, then each time we classify a new object, we need to sample $21$ negative objects out of those $50$. Another option would be to bootstrap the $p_c$-value.

Finally, if a lot of data are available after training, we can fit the class-conditional distributions of $T$ and calculate the $p$-values as the corresponding integrals analytically. Since the scores are one-dimensional, densities should not be hard to estimate.
\section{Empirical Evaluation}\label{sec:experiments}
We experimented with the neural network from \cite{Simic2020}, designed to classify a distribution as normal or non-normal by inspecting a small sample drawn from it. In the original study, the network outperformed the standard tests of normality by many metrics. It had the overall accuracy of about $95\%$, $AUROC$ almost equal to $1$, and very high power ($1-FPR$) in detecting non-normal distributions, substantially higher than the tests. However, its $FPR$ and $FNR$ could not be controlled, which is in this experiment by applying our framework. %The code is available as the supplementary material.

First, we visualized the class-conditional densities of the network's scores in Figure \ref{fig:densities_of_T_NN}. From there, it can be seen that the assumption about higher values being more in line with the hypothesis $y(x) = 1$ holds. We calculated the $p_c$-values using large sets containing $32625$ normal and $32625$ non-normal samples (with $10, 20, \ldots, 100$ elements) as $\mathcal{Z}_1$ and $\mathcal{Z}_0$. The non-normal samples were simulated from the Pearson family of distributions by randomly selecting the first four moments. The normal samples were generated from normal distributions with randomly determined mean and standard deviation. We evaluated the performance of thus obtained neural tests of normality on another large set of $13100$ normal and $13100$ non-normal samples with $10, 20, \ldots, 100$ elements. The results for the network and standard tests of normality with $\alpha=0.01, 0.05$ are presented in Table \ref{tab:poredjenje_D}. Note that for each $\alpha$ we were able to obtain two neural tests of normality from the original network: $NN_0 \equiv \one{p_0(x)\leq \alpha}$ with $FPR \leq \alpha$, and $NN_1 \equiv \one{p_1(x)\leq \alpha}$ with $FNR \leq \alpha$. The standard statistical tests cannot control $FPR$, so this is a clear advantage of the neural test $NN_0$.

Additionally, we plotted the values of $\alpha$ against empirical $FPR$ of $NN_0$ and $FNR$ in Figures \ref{fig:NN_p0_FPR} and \ref{fig:NN_p1_FNR}. We see that choice of $\alpha$ effectively sets the target error rate to $\alpha$. Moreover, the probability to correctly classify samples of class $1-c$ increases with the sample size for $NN_c$, which is illustrated in Figures \ref{fig:NN} and \ref{fig:NNN} for $\alpha = 0.05$.
\setlength{\tabcolsep}{3pt}
\begin{small}
\begin{table}
	\centering
	\caption{$NN$ is the original neural network from \cite{Simic2020}, $NN_c$ is the neural test derived from it by bounding $FNR$ ($c=1$) or $FPR$ ($c=0$) by $\alpha$ from above.  SW, LF, JB, and AD stand for the Shapiro-Wilk, Lilliefors, Jarque-Berra, and Anderson-Darling tests of normality, respectively, and $A$ for accuracy.}
\label{tab:poredjenje_D}
	\begin{tabular}{llllllllllllll}
		& AD      &         & $NN_0$  &         & $NN_1$  &         & JB      &         & LF      &         & SW      &         & NN      \\ \toprule
		$\alpha$ & $0.01$  & $0.05$  & $0.01$  & $0.05$  & $0.01$  & $0.05$  & $0.01$  & $0.05$  & $0.01$  & $0.05$  & $0.01$  & $0.05$  &         \\ 
		$A$      & $0.804$ & $0.827$ & $0.896$ & $0.912$ & $0.840$ & $0.890$ & $0.813$ & $0.831$ & $0.760$ & $0.794$ & $0.821$ & $0.840$ & $0.911$ \\ 
		$FNR$    & $0.009$ & $0.049$ & $0.198$ & $0.104$ & $0.010$ & $0.049$ & $0.015$ & $0.032$ & $0.010$ & $0.046$ & $0.009$ & $0.049$ & $0.128$ \\
		$FPR$    & $0.382$ & $0.296$ & $0.010$ & $0.053$ & $0.310$ & $0.171$ & $0.359$ & $0.305$ & $0.470$ & $0.367$ & $0.349$ & $0.271$ & $0.05$  \\ \bottomrule
		\end{tabular}
\end{table}
\end{small}
\begin{figure}
	\centering
	\subfloat[Class conditional densities of $T$\label{fig:densities_of_T_NN}]{\includegraphics[width = 1.8in]{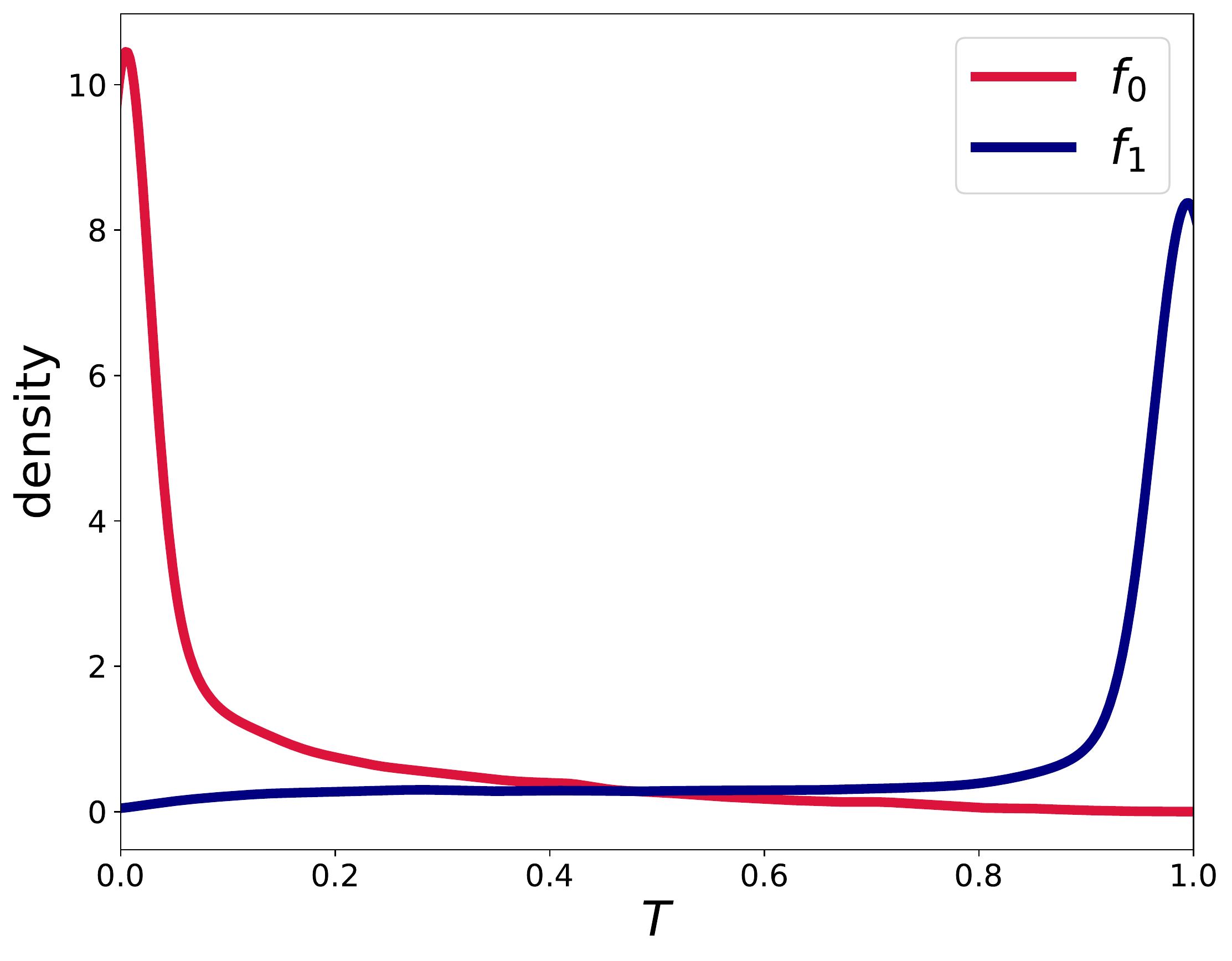}} 
	\subfloat[$NN_0$ controls $FPR$\label{fig:NN_p0_FPR}]{\includegraphics[width = 1.8in]{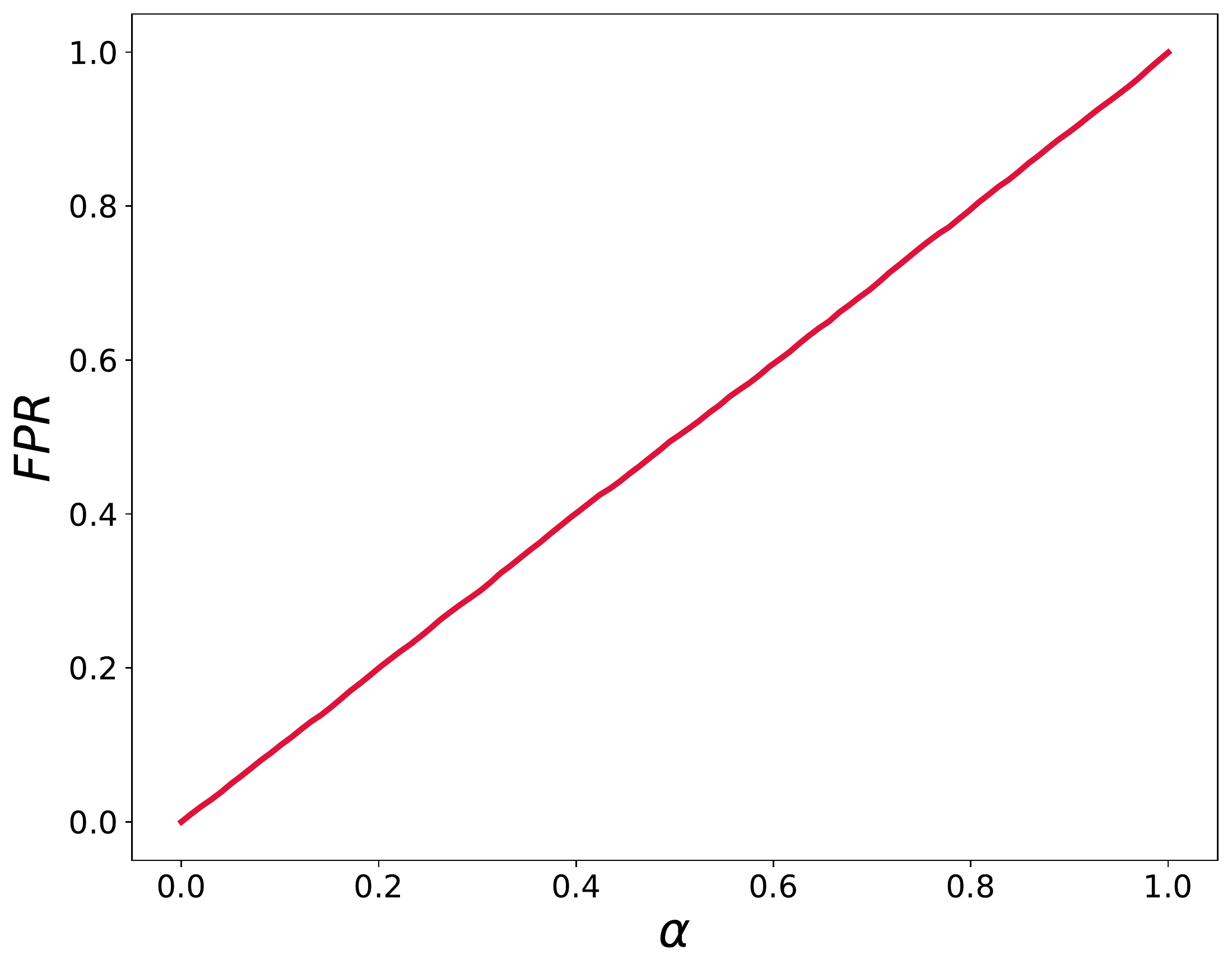}}  
	\subfloat[$NN_1$ controls $FNR$\label{fig:NN_p1_FNR}]{\includegraphics[width = 1.8in]{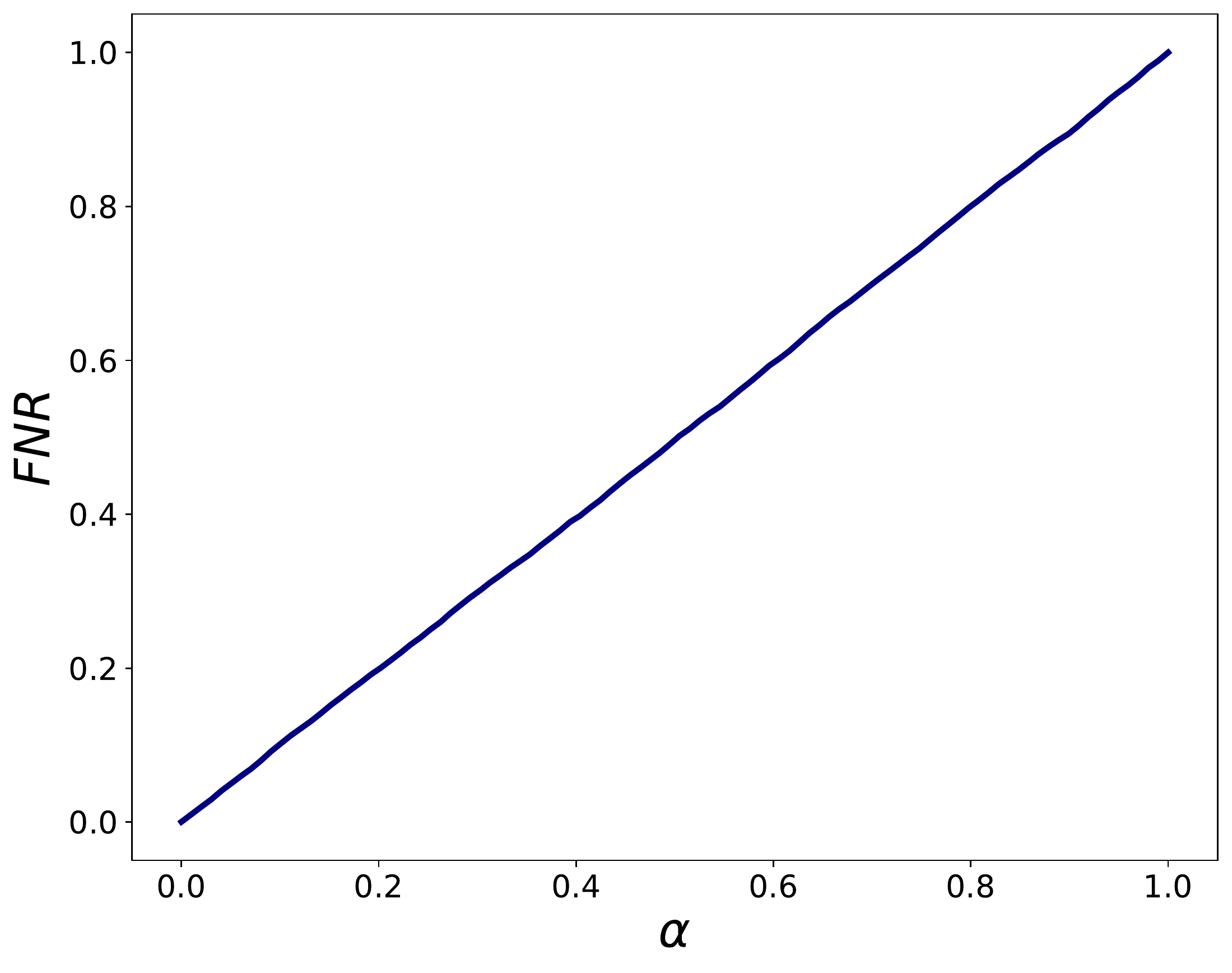}} \\
	\subfloat[$TPR$ of $NN_0$ grows with $n$\label{fig:NN}]{\includegraphics[width = 1.8in]{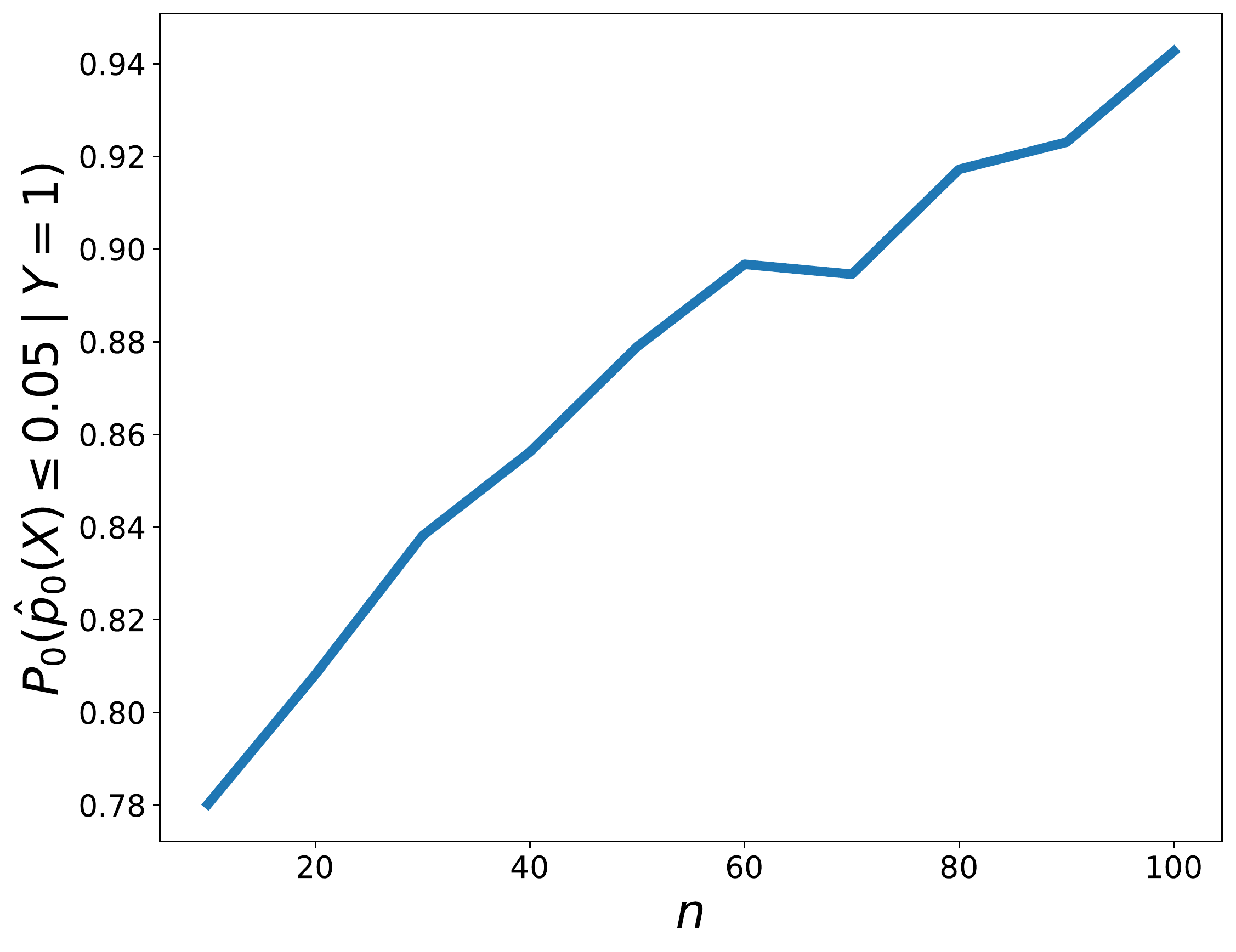}} 
	\subfloat[$TNR$ of $NN_1$ grows with $n$\label{fig:NNN}]{\includegraphics[width = 1.8in]{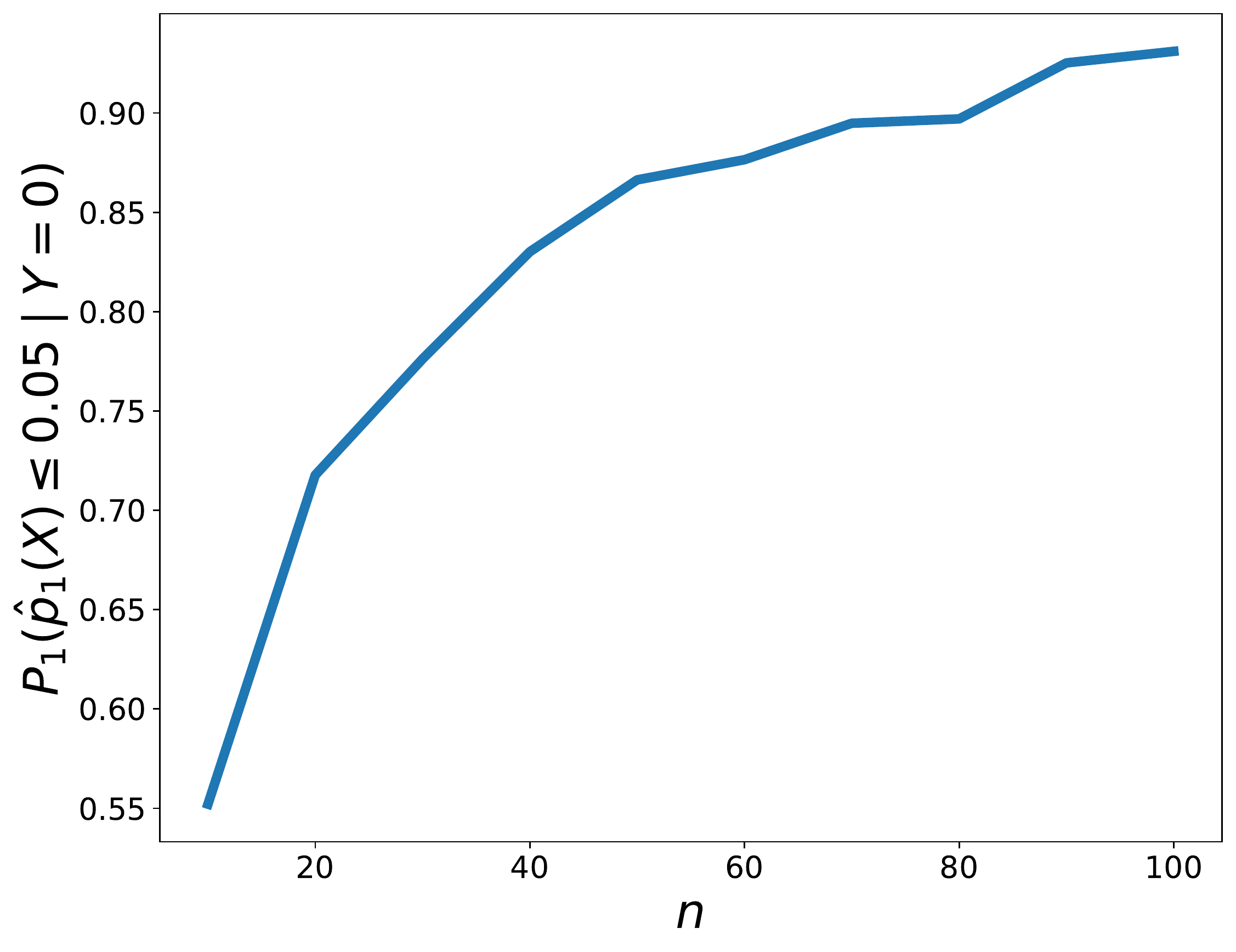}} 
	\caption{Analysis of $NN_0$ and $NN_1$, the neural tests of normality. $n$ is the sample size. $TNR=1-FPR$ is the true negative rate, and $TPR=1-FNR$ is the true positive rate.}
	\label{fig:NN_results}
\end{figure}

Finally, we evaluated the tests on four groups of distributions that are commonly used in statistical literature to estimate the power of a normality test. The distributions are summarized in Table \ref{tab:alternative_distributions} and the results are presented in Figure \ref{fig:tnr_po_grupama}. We see that $NN_1$ is the best among the $FNR$-controlled tests, but that its power in group $G_4$ for the samples with more than $50$ elements is unstable and drops for the samples with more than $70$ elements. However, $NN_1$ performed better than standard statistical tests in other groups and for the samples in $G_4$ with $50$ elements at most. The results confirm that $FPR$ of $NN_0$ can be controlled, the only exception being the samples with $90$ and $100$ elements in group $G_4$. This is due to the fact that the distributions in group $G_4$ are different from those used during training. Therefore, to fully incorporate the neural tests of normality in everyday use of statistics, a bigger and more diverse training dataset is needed, and potentially, the networks should be deeper, as the one used here has only two hidden layers. Still, it managed to surpass all the standard tests of normality. We used small samples only because the standard tests are sufficiently powerful when the samples are large enough.

\begin{table}
	\centering
	\caption{Specification of the non-normal distributions in dataset $\mathcal{C}$. $\Gamma$ -- Gamma distribution, $Exp$ -- exponential, $LN$ -- lognormal, $W$ -- Weibull, $U$ -- uniform, $B$ -- Beta.}
	\label{tab:alternative_distributions}
	\begin{tabular}{llll}
		\noalign{\smallskip}\hline\noalign{\smallskip}
		Group &  Support & Shape & Distributions\\
		\noalign{\smallskip}\hline\noalign{\smallskip}
		$G_1$ & $(-\infty, \infty)$ & Symmetric & $t(1)$, $t(3)$, standard logistic, standard Laplace \\
		\midrule
		$G_2$ & $(-\infty, \infty)$ & Asymmetric & Gumbel($0, 1$), Gumbel($0, 2$), Gumbel($0, 1/2$) \\
		\midrule
		$G_3$ & $(0,\infty)$ & Various  &  $Exp(1)$, $\Gamma(1, 2)$, $\Gamma(1, 1/2)$, $LN(0, 1)$, $LN(0, 2)$ \\
		\multicolumn{3}{l}{ } & $LN(0, 1/2)$, $W(1, 1/2)$, $W(1, 2)$ \\
		\midrule
		$G_4$ & $(0, 1)$ & Various & $U[0, 1]$, $B(2, 2)$, $B(1/2,1/2)$, $B(3,3/2)$, $B(2,1)$ \\
		\noalign{\smallskip}\hline
	\end{tabular}
\end{table}

\begin{figure}
	\subfloat[$G_1$]{\includegraphics[width = 2.5in]{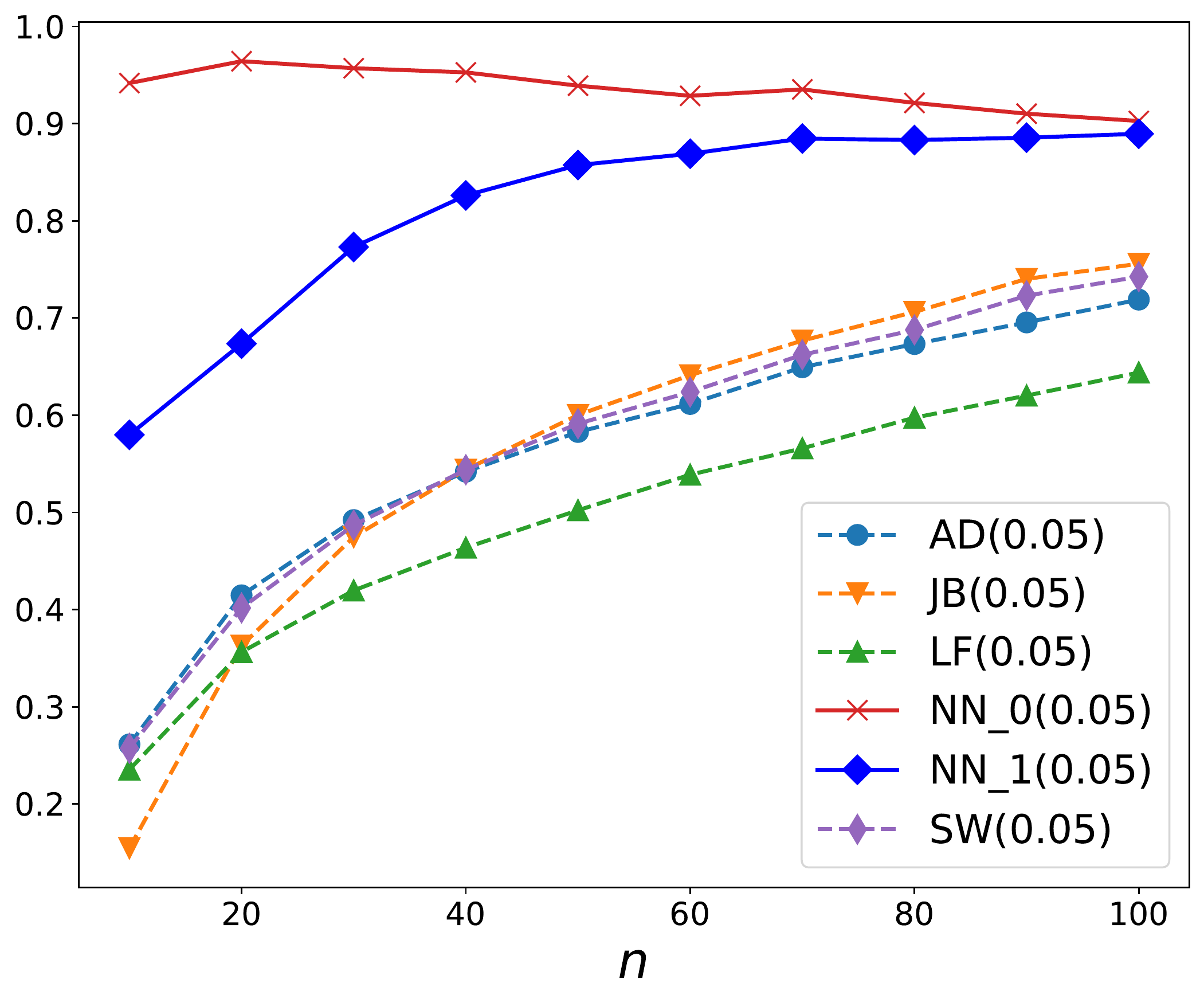}} 
	\subfloat[$G_2$]{\includegraphics[width = 2.5in]{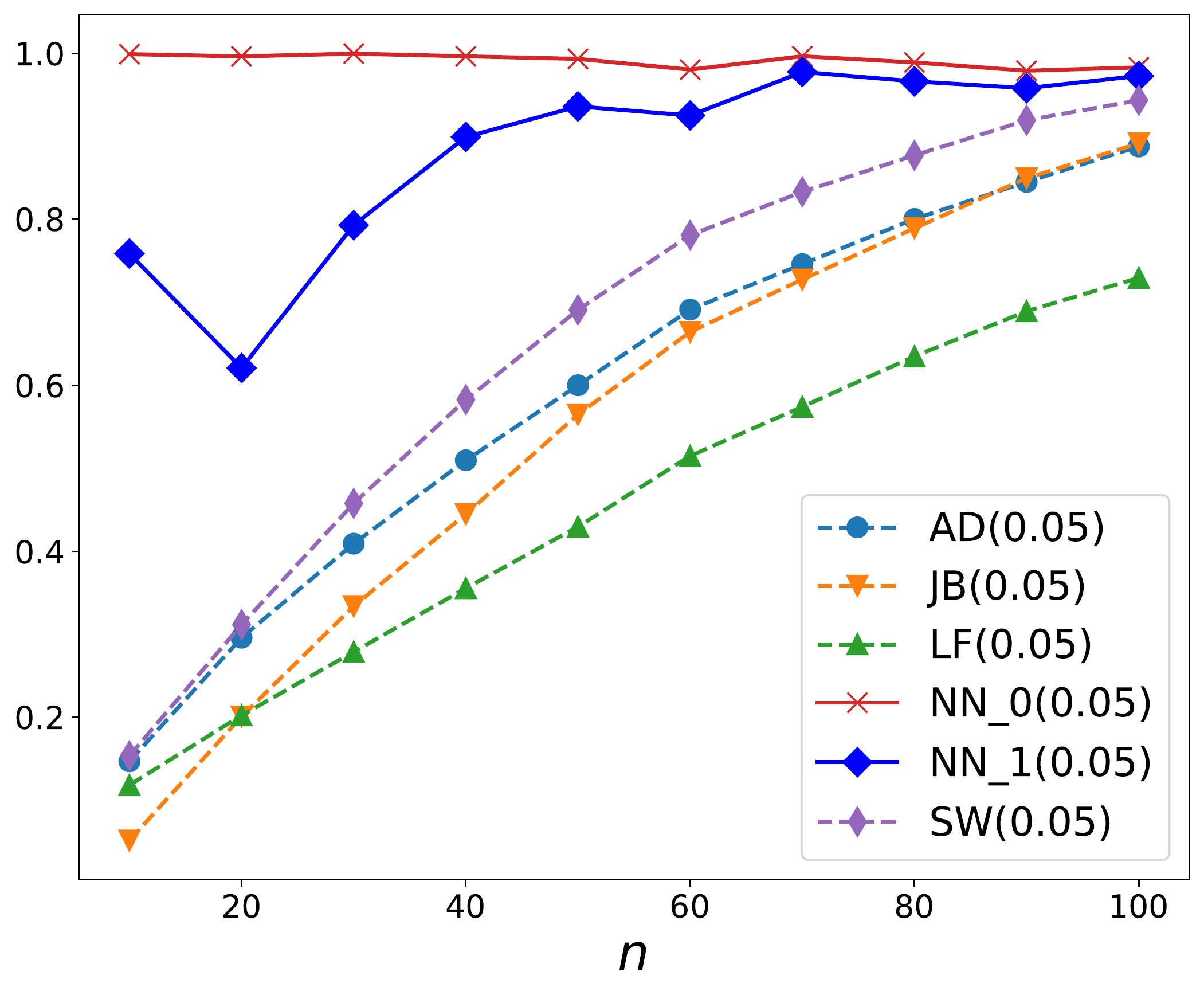}} \\
	\subfloat[$G_3$]{\includegraphics[width = 2.5in]{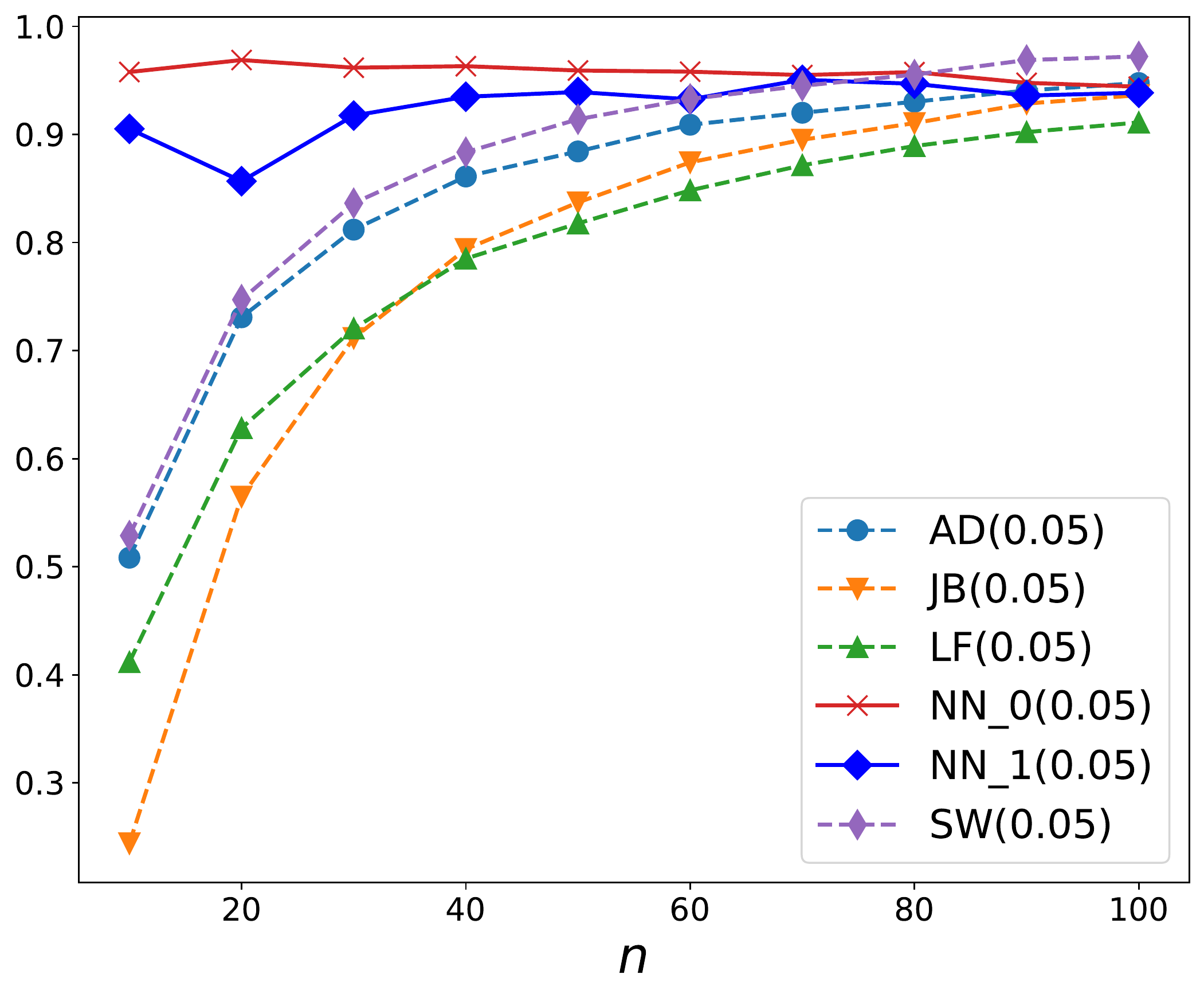}} 
	\subfloat[$G_4$]{\includegraphics[width = 2.5in]{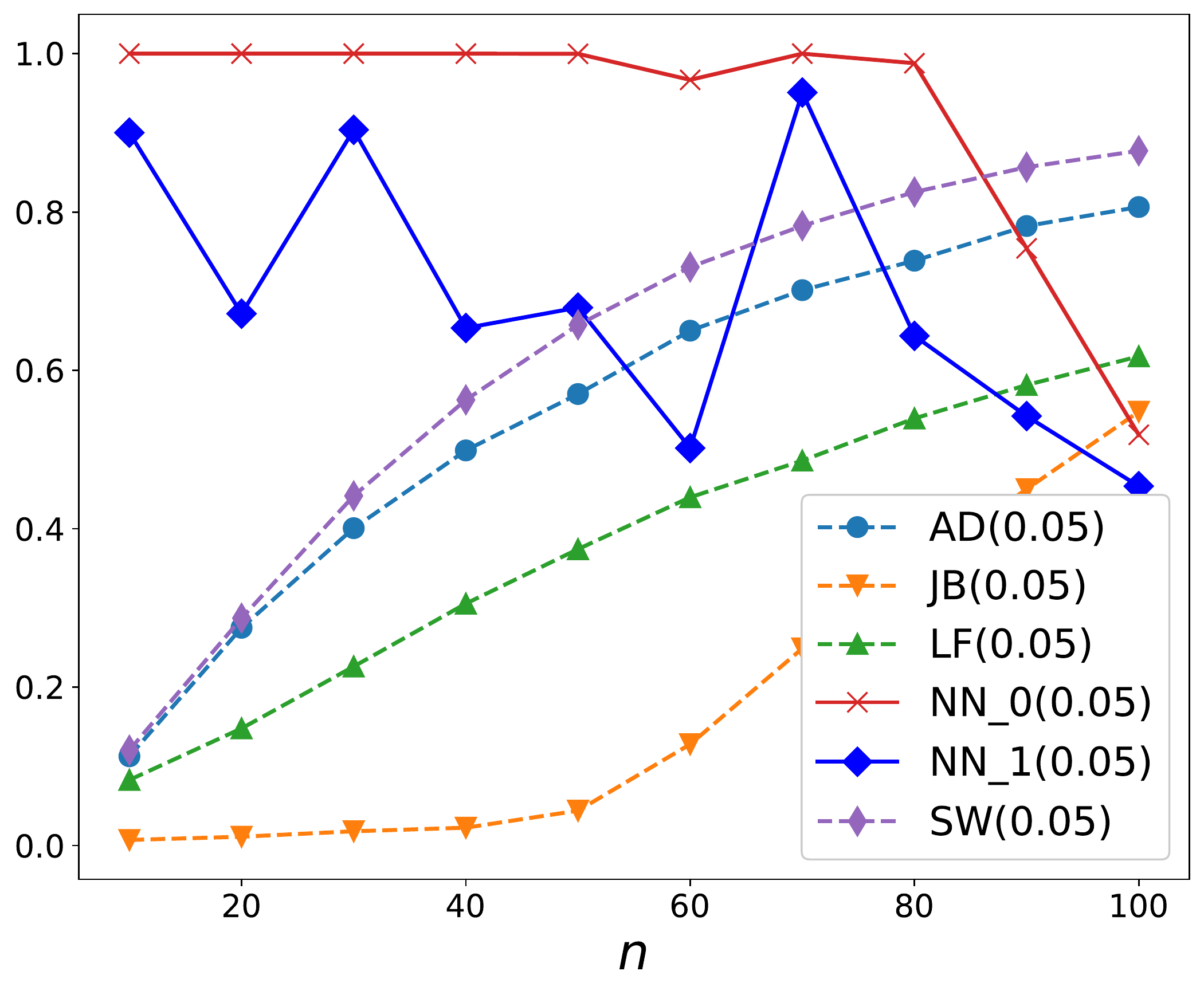}}  
	\caption{The power ($1-FPR$) of $NN_1$ and the standard tests of normality on different groups of non-normal distributions from set $\mathcal{C}$. The $FNR$ limit was set to $\alpha=0.1$. Note that for $NN_0$, $\alpha$ controls $FPR$, not $FNR$, and the plotted values should be close to $1-FPR$ for each group and sample size.}
	\label{fig:tnr_po_grupama}
\end{figure} 

\section{Discussion and Conclusions}\label{sec:discussion_and_conclusion}
In this study, we present a general framework to turn a binary classifier into two statistical tests whose class-conditional error rates can be kept at the desired levels. The derived tests use classification $p$-values to both control the error rates and quantify statistical uncertainty of their decisions. The only assumption is that the classifier at hand has a scoring function whose higher scores are more common for the objects of the positive class, while negative objects tend to have lower scores. It can also be the other way around, as long as the scores of the objects belonging to the different classes tend to group together at the opposite ends of the score range. The formulated framework was successfully applied to the neural networks \cite{Simic2020} that were developed to classify distributions as normal or non-normal by inspecting small samples from it. The neural tests of normality that we derived proved capable of keeping the error rates at chosen levels, empirically verifying the framework's theory. The theorems in this paper, except for Theorem \ref{th:ogranicenost}, are not new. They are collections of mathematical results that were already there, but which we formulated in the context of binary classification. This should not come as a surprise, since the theory of statistical hypothesis testing is well-established, and once we see how it is connected to binary classification, its theoretical results are ready to use in the new context.

The proposed framework has certain advantages over similar methods that have been proposed in literature. It can be applied to the classifiers which have already been trained and are already in use. No changes to the process of training are needed. If the scores of the held-out objects used to calculate the $p$-values are sorted prior to use, calculation will be of logarithmic complexity. To calculate $p$-values, no additional classifiers have to be trained. Furthermore, by working with one-dimensional scores, the framework avoids density estimation in high-dimensional spaces. The cost-sensitive classification, empirical and plug-in Neyman-Pearson approaches, cross-validates $p$-values, and typicality indices have at least one of the mentioned shortcomings.

On the other hand, a drawback of our framework is that it requires a separate dataset to compute the $p_c$-values, which leaves less data for training and cross-validation. Theorem \ref{th:ogranicenost}, however, shows that if we want to have $FPR$ or $FNR$ lower than $\alpha\in(0, 1)$ in the long run, it is sufficient to sample $\lceil1/\alpha\rceil$ scores from the corresponding class whenever a new object is to be classified. Another shortcoming is that $p$-values are usually misinterpreted \cite{Sellke2001,Berger2003,Goodman2008,Nuzzo2014,Fraser2016,Demidenko2016,Wasserstein2016,Wasserstein2019}, so special care must be given to interpret them correctly. Finally, there are no guarantees that, even though one error rate is controlled, the other is this way minimized or kep close to the minimal possible.

The framework that we proposed here is universal. It can be applied to any binary  classification problem where one of the errors is more severe and we want to limit its probability (i.e. its long-run frequency). That is an important requirement in many cases, especially in applications of machine-learning classifiers in medicine. Even though there are still some issues to resolve and directions to investigate, we believe that this framework may bring benefit and prove interesting to machine-learning researches, statisticians and the end users of binary classifiers.

Some ideas for future research are:
\begin{itemize}
	\item Combining classifiers with different schools of hypothesis testing to obtain new schools of classification, that have different assumptions, methods, and ways to interpret the results.
	\item Generalizing the framework to classification with multiple classes.
	\item Proving new theoretical results. For example, how does breaking the assumption that higher scores are more in line with class $1$ affects the capability of the framework to control $FPR$ and $FNR$, and so on.
\end{itemize}

\newpage
\section*{Broader Impact}

As a general methodology that can control the long-run frequency of the selected error in binary classification, the framework proposed in this study can be applied in all the fields of science and industry. The typical use case would be the one in which there are many decisions to be made between two alternatives, and the end goal is to make sure that one of the alternatives is \textbf{not} incorrectly chosen in more than the predefined percentage of time (in the long run). This scenario encompasses both hypothesis testing and classification, as, broadly speaking, a decision can be understood as choosing one hypothesis over another, or assigning one label over the other. The examples are: medicine, in which we want classifiers that do not miss a condition in more than, say, $1\%$ or $2\%$ of the cases, or financial industry, where an automated recommendation system that is guaranteed to be wrong only $1\%$ of the time when giving a recommendation to invest would be highly appreciated by the banks and investment funds.

As for the neural tests, many statistical techniques require that the data be normal in order to give valid inference. The neural tests that we developed in this study by applying the framework performed better than the standard statistical tests of normality. Having a more accurate normality test would be of great importance to statisticians and scientists who test for normality on a daily basis to choose the appropriate way to analyze their data, or verify that the residuals of a regression model are normally distributed in order to confirm that the model is unbiased. Better accuracy of such a fundamental tool in science would have an impact on society in general, as any advance in scientific methodology is an opportunity for technological as well.

A potential drawback of our framework is that the classification $p_c$-values, that are central to it, could be misinterpreted by the end users. In standard statistics, many logical fallacies with $p$-values have appeared over time, mainly because the definition of the $p$-values can be difficult to grasp intuitively. We tried to give intuitive explanations and motivation in addition to mathematical details to make sure that the classification $p_c$-values do not get misinterpreted. One must always have their frequentist nature in mind: the $p_c$-values are the long-run frequencies of the corresponding errors, and by using $\alpha\in(0, 1)$ as the decision threshold, we can make sure that, \textbf{in the long run}, we will not be committing the target error more than $100\alpha\%$ of time.

If the data sets reserved for calculating the classification $p_c$-values are not random and are not representative of the general population of objects on which we want to control the target error's frequency, it may not behave as we want to in the long run. For example, the sub-populations that are extremely under-represented in the datasets used, or not represented at all, will have an unpredictable error rate. However, if a dataset is not sufficiently large nor representative, any statistical model's capacities for valid inference would be limited for the dataset in question.

\newpage
\bibliography{references}

\newpage
\appendix
\section{Proofs}
\subsection{Proof of Theorem \ref{th:p_vrednost_je_uniformna}}
For $y(X)=c=1$, $p_c(X)$ is the CDF of $T(X) \mid y(X)=1$. By the probability integral transform \citep{Angus1994}, $p_1(X)$ is uniform over $[0, 1]$. For $y(X)=c=0$, $p_0(X)$ is the difference between $1$ and the CDF of  $T(X) \mid y(X)=0$, so it is uniform over $[0, 1]$.
\subsection{Proof of Theorem \ref{th:dobra_svojstva_aproksimatora}}
\begin{enumerate}
	\item This property is a direct consequence of the linearity of expectation.
	\item Property \eqref{eq:varijabilnost}  follows from  $V[aZ_1+bZ_2]=a^2V[Z_1]+b^2V[Z_2]$ for independent variables $Z_1$ and $Z_2$, the fact that the indicator variables $\bm{1}\left\{T(X_i)\,\rho\,T(x)\right\}$ ($\rho \in \{\leq, \geq\}$) are i.i.d, and because the maximum of $f(u) = u(1-u)$ ($f:[0,1] \rightarrow [0, 1]$) is $1/4$.
	\item Almost sure convergence was proven by the Glivenko--Cantelli theorem (see \cite{Tucker1959}).
	\item By Fatou's Lemma \cite{Simons1995}, consistency follows from almost sure convergence.
	\item Formula \eqref{eq:intervali_poverenja} is a direct consequence of the Dvoretzky--Kiefer--Wolfowitz inequality \cite{Dvoretzky1956,Massart1990}. 
\end{enumerate}
\subsection{Proof of Theorem \ref{th:ogranicenost}}
We will prove the theorem for $c=1$ as the proof for $c=0$ is completely analogous. To do so, we will need Lemma \ref{le:gornje_granice_za_vece_ili_jednako}.
\begin{lemma}\label{le:gornje_granice_za_vece_ili_jednako}
	In any set with $n$ numbers, for each $k=0,1,\ldots,n-1$  there can be at most $k+1$ numbers in the set that are greater than or equal to at most $k$ other numbers in the set. Formally, for each set $A=\{a_i\}_{i=1}^{n}$, $a_i\in A$ and $m(a_i)=\lvert\left\{j \in \{1,2,\ldots,n\}\setminus\{i\} \colon a_j \leq a_i\right\}\rvert$ the following holds for each $k=0,1,\ldots,n-1$:
	\begin{equation}
	r_k=\left\lvert\left\{a_i \in A \colon m(a_i) \leq k\right\}\right\rvert \leq k
	\end{equation} 
\end{lemma}
\begin{proof}
	Sort $A$ in a non-decreasing array
	\begin{equation}\label{eq:nerastuci_niz}
	a_1^{(s)} \leq a_2^{(s)} \leq \ldots \leq a_i^{(s)} \leq \ldots a_n^{(s)}
	\end{equation}
	Obviously
	\begin{equation}\label{eq:donja_granica_m_ai}
	m(a_i^{(s)}) \geq i - 1 \iff i \leq m(a_i^{(s)}) + 1
	\end{equation}
	for each $i=1,2,\ldots, n$ because there is at least $i-1$ other numbers in $A$ that are not greater than $a_i^{(s)}$. 
	
	Let $k$ be any number from $\{0,1,2,\ldots,n-1\}$.  For each $a \in A$ for which $m(a) \leq k$, from \eqref{eq:donja_granica_m_ai} it holds that $a$ must come before position $k+1$, including it as well, in the sorted array \eqref{eq:nerastuci_niz}. Let $r_k$ be the number of such numbers. Suppose $r_k>k+1$. Then, at least one of those numbers, denoted as $a'$, would have to come after position $k+1$ in the sorted array. It would then hold that $m(a') \geq k + 1$, which is a contradiction. Hence, it must be that $r_k \leq k + 1$.
\end{proof}
We now prove the main result. The proof is an adaptation of the proof for the $p$-values from \cite{Dumbgen2008} being stochastically larger than $U[0, 1]$.

	Let $X_0=X$. Let's first note that Lemma \ref{le:gornje_granice_za_vece_ili_jednako} holds for each realization of the variables $T(X_0),T(X_1), \ldots, T(X_n)$. So, we conclude that the Lemma applies to the variable $R_k=\lvert\{i \in {0,1,\ldots,n} : m(X_i) \leq k\}\rvert \leq k$. 
	Note that $\{T(X_i)\}_{i=0}^{n}$ are i.i.d., as well as $\{m(X_i)\}_{i=0}^{n}$.
	
	Also, note that $\hat{p}_1(X)=\frac{1}{n}m(X_0)$. Let $k=\floor{\alpha n} = \alpha n - \delta\leq \alpha n$ ($\delta \in [0, 1)$). Then
	\begin{align*}
	\Prob\{\hat{p}_c(X) \leq \alpha\} &= \Prob\{m(X_0) \leq \alpha n\} \\
	&= \Prob\{m(X_0) \leq k\} \tag{since $m(X_0)$ is discrete}\\
	&=\frac{1}{n}n\Prob\{m(X_0) \leq k\}\\
	&=\frac{1}{n}\sum_{i=1}^{n}\Prob\{m(X_i) \leq k\} \tag{$m(X_0),\ldots, m(X_n)$ are i.i.d.}\\
	&=\frac{1}{n}\sum_{i=1}^{n}\Exp\left[\bm{1}\left\{m(X_i)\leq k\right\}\right]\\
	&=\frac{1}{n}\Exp\left[R_k-1\right]\\
	&\leq \frac{1}{n}(k+1-1) \\
	&=\frac{1}{n}(\alpha n - \delta)\\
	&=\alpha - \frac{\delta}{n} \\
	&\leq \alpha \tag{$\delta \geq 0$} \\
	\end{align*}
	which we wanted to prove.
\end{document}